\documentclass{article}[11pt]
\usepackage[utf8]{inputenc}
\usepackage{amssymb,amsmath}
\usepackage{booktabs, multirow}
\usepackage[title,titletoc]{appendix}
\usepackage{natbib}
\usepackage{setspace}
\usepackage{algorithm}
\usepackage{algorithmic}

\usepackage[margin=1in]{geometry}

\usepackage[
            CJKbookmarks=true,
            bookmarksnumbered=true,
            bookmarksopen=true,
            colorlinks=true,
            citecolor=red,
            linkcolor=blue,
            anchorcolor=red,
            urlcolor=blue,
            ]{hyperref}
\usepackage[usenames,dvipsnames]{xcolor}
\usepackage{dsfont}
\usepackage{amsbsy}
\usepackage{amssymb}
\usepackage{amsmath}%
\usepackage{amsfonts}%
\usepackage{amsthm}
\usepackage{breqn}
\usepackage{graphicx}
\usepackage{colonequals}
\usepackage{psfrag}
\usepackage{subfig}
\usepackage{mathrsfs}
\usepackage{hyperref}
\usepackage{mdframed}
\usepackage{enumerate}
\usepackage{enumitem}
\usepackage{mathtools}

\allowdisplaybreaks

\newtheorem{lem}{Lemma}

\newtheorem{prop}{Proposition}

\makeatletter
\newtheorem*{rep@theorem}{\rep@title}
\newcommand{\newreptheorem}[2]{%
\newenvironment{rep#1}[1]{%
 \def\rep@title{#2 \ref{##1}}%
 \begin{rep@theorem}}%
 {\end{rep@theorem}}}
\makeatother
\newreptheorem{conj}{Conjecture}

\theoremstyle{definition}

\newcommand{\calX}{\mathcal{X}}

\newcommand{\calB}{\mathcal{B}}
\newcommand{\calD}{\mathcal{D}}

\newcommand{\calY}{\mathcal{Y}}

\renewcommand{\tilde}{\widetilde}

\newcommand{\Xh}{\hat{X}}

\newcommand{\Yh}{\hat{Y}}

\newcommand{\Reals}{\mathbb{R}}

\newcommand{\defined}{\triangleq}

\newcommand{\ExpVal}[2]{\mathbb{E}\left[ #2 \right]}

\newcommand{\Yt}{\tilde{Y}}
\newcommand{\yt}{\tilde{y}}

\newcommand{\sto}{\mbox{\normalfont s.t.}}

\newcommand{\EE}[1]{\ExpVal{}{#1}}

\newcommand{\xh}{\hat{x}}

\newcommand{\yh}{\hat{y}}

\newcommand{\ah}{\hat{a}}
\newcommand{\eh}{\hat{e}}

\newcommand{\DKL}{D_\mathsf{KL}}

\definecolor{light-gray}{gray}{.90}
\definecolor{aliceblue}{rgb}{0.94, 0.97, 1.0}
\definecolor{airforceblue}{rgb}{0.36, 0.54, 0.66}

\newmdenv[%
  backgroundcolor=light-gray, %
  linecolor=black,
  linewidth =1pt,%
  skipabove = 10pt,%
  skipbelow = 10pt
]{comments}

\newmdenv[%
  backgroundcolor=aliceblue, %
  linewidth = 2pt,%
  skipabove = 10pt,%
  skipbelow = 10pt,
  pstrickssetting={linestyle=dashed,},
  linecolor=airforceblue,
  middlelinewidth=2pt
]{TODO}

\begin{document}

\title{Optimized Data Pre-Processing for Discrimination Prevention}
\author{Flavio P. Calmon, Dennis Wei, Karthikeyan Natesan Ramamurthy, and Kush R. Varshney\\Data Science Department, IBM Thomas J. Watson Research Center\thanks{Contact: \texttt{\{fdcalmon,dwei,knatesa,krvarshn\}@us.ibm.com}}}
\date{}
\maketitle






\begin{abstract} 
Non-discrimination is a recognized objective in algorithmic decision making. In this paper, we introduce a novel probabilistic formulation of data pre-processing for reducing discrimination. We propose a convex optimization for learning 
a data transformation with three goals: controlling discrimination, limiting distortion in individual data samples, and preserving utility. We characterize the impact of limited sample size in accomplishing this objective,  and apply two  instances of the proposed optimization to datasets, including one on real-world criminal recidivism. The results demonstrate that all three criteria can be simultaneously achieved and also reveal interesting patterns of bias in American society.
\end{abstract} 

\section{Introduction}
\label{sec:intro}
Discrimination is the prejudicial treatment of an individual based on membership in a legally protected 
group such as a race or gender.  Direct discrimination occurs when protected attributes are used explicitly in making decisions, which is referred to as \emph{disparate treatment} in law.  More pervasive 
nowadays is indirect discrimination, in which protected attributes are not used but reliance on variables correlated with them leads to significantly different outcomes for different groups.  The latter phenomenon is termed \emph{disparate impact}.  Indirect discrimination may be intentional, as in the historical practice of ``redlining'' in the U.S.\ in which home mortgages were denied in zip codes populated primarily by minorities.  However, the doctrine of disparate impact applies in many situations regardless of actual intent. 

Supervised 
learning algorithms, increasingly used for decision making in applications of consequence, may at first be 
presumed to be fair and devoid of inherent bias, but in fact, inherit any bias or discrimination present in the data on which they are trained \citep{calders2013unbiased}.  Furthermore, simply removing protected variables from the data is not enough since it does nothing to address indirect discrimination and may in fact conceal it. The need for more sophisticated tools has made  
discrimination discovery and prevention 
an important research area \citep{pedreschi2008discrimination}. 
 
Algorithmic discrimination prevention involves 
modifying one or more of the following to ensure that decisions made by supervised learning 
methods are less biased: 
(a) the training data, (b) the learning algorithm, and (c) the ensuing decisions themselves. These are respectively classified as pre-processing \citep{hajian2013simultaneous}, in-processing \citep{fish2016confidence, zafar2016fairness, kamishima2011fairness} and post-processing approaches \citep{hardt2016equality}. In this paper, we focus on pre-processing since it is the most flexible in terms of the data science pipeline: it is independent of the modeling 
algorithm and can be integrated with data release and publishing mechanisms. 

Researchers have also studied several notions of discrimination and fairness. Disparate impact is addressed by the principles of \emph{statistical parity} and \textit{group fairness} \citep{feldman2015certifying}, which seek similar outcomes for all groups. 
In contrast, \textit{individual fairness} \citep{dwork2012fairness} mandates that similar individuals be treated similarly irrespective of group membership.  For classifiers and other predictive models, equal error rates for different groups are a desirable property \citep{hardt2016equality}, as is calibration or lack of \textit{predictive bias} in the predictions \citep{zhang2016identifying}. The tension between the last two notions is described by \citet{kleinberg2017inherent} and \citet{chouldechova2016fair}; the work of \citet{friedler2016possibility} is in a similar vein. \citet{corbett2017algorithmic} discuss the cost of satisfying prevailing notions of algorithmic fairness from a public safety standpoint and discuss the trade-offs. Since the present work pertains to pre-processing and not modeling, balanced error rates and predictive bias are less relevant criteria.  Instead we focus primarily on achieving group fairness while also accounting for individual fairness through a distortion constraint.

Existing pre-processing approaches include sampling or re-weighting the data to neutralize discriminatory effects \citep{kamiran2012data}, changing the individual data records \citep{hajian2013methodology}, and using $t$-closeness \citep{li2007t} for discrimination control \citep{ruggieri2014using}.  A common theme is the importance of balancing discrimination control against 
utility of the processed data.  However, this prior work neither presents general and principled optimization frameworks for 
trading off these two criteria, nor allows connections to be made to the broader statistical learning and information theory literature via probabilistic descriptions. 
Another shortcoming 
is that individual distortion or fairness is not made explicit. 

\begin{figure}[!tb]
    \centering
    \includegraphics[scale = 0.35]{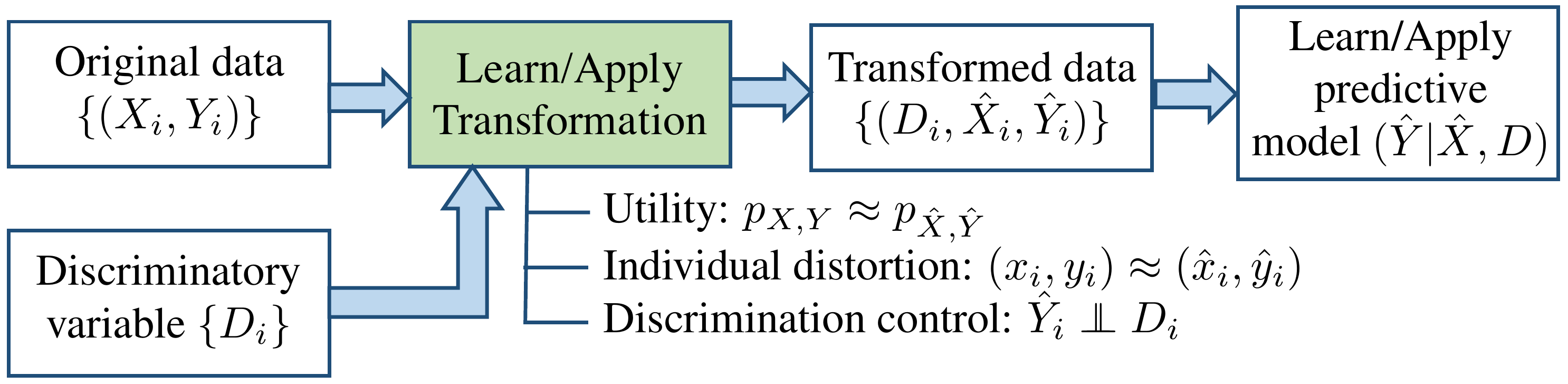}
    \caption{\small{The proposed pipeline for predictive learning with discrimination prevention. 
    \textit{Learn} mode applies with training data and \textit{apply} mode with novel test data. Note that test data also requires transformation 
    before predictions can be obtained.
    }}
    \label{fig:disc_transform_pipeline}
\end{figure}

In this work, addressing gaps in the pre-processing literature, 
we introduce a probabilistic framework for discrimination-preventing pre-processing in supervised learning.  Our aim in part is to work toward a more unified view of previously proposed concepts and methods, which may help to suggest refinements.  We formulate the determination of a pre-processing transformation as an optimization problem 
that trades off discrimination control, data utility, and individual distortion. 
(Trade-offs among various fairness notions may be inherent as 
shown by \citet{kleinberg2017inherent}.) 
While discrimination and utility are defined 
at the level of probability distributions, distortion is controlled 
on a per-sample basis, thereby limiting the effect of the transformation on individuals and ensuring a degree of individual fairness. Figure \ref{fig:disc_transform_pipeline} illustrates the supervised learning pipeline that includes our proposed discrimination-preventing pre-processing. 

The work of \citet{zemel2013fairrepresentations} is closest to ours in 
also presenting a framework with three criteria related to discrimination control (group fairness), individual fairness, and utility.  However, the criteria are manifested less directly than in our proposal.  In particular, discrimination control is posed in terms of intermediate features rather than outcomes, individual distortion does not take outcomes into account (simply being an $\ell_2$-norm between original and transformed features), and utility is specific to a particular 
classifier.  Our formulation more naturally and generally encodes these fairness and utility desiderata. 

Given the novelty of our formulation, we devote more effort than usual to discussing its motivations and potential variations.  We state natural conditions under which the proposed optimization problem is convex.  The resulting transformation is in general a randomized one. 
The proposed optimization problem assumes as input an estimate of the distribution of the data which, in practice, can be imprecise due to limited sample size. Accordingly, we characterize the possible degradation in discrimination 
and utility guarantees at test time in terms of the training sample size. 
As a  demonstration of our framework, we apply specific instances of it to a prison recidivism risk score dataset \citep{compas_risk_scores} and the UCI adult dataset \citep{Lichman:2013}. By  solving the optimization problem, we show that discrimination, distortion, and utility loss can be controlled simultaneously with real data. In addition, the resulting transformations reveal intriguing demographic patterns in the data.



\vspace{-.15in}
\section{General Formulation}
\label{sec:proposed}
We are given a dataset consisting of $n$ i.i.d.\ samples $\left\{(D_i,X_i,Y_i)\right\}_{i=1}^n$ from a joint distribution $p_{D,X,Y}$ with domain $\calD \times \calX \times \calY$.  
Here $D$ denotes one or more \textit{discriminatory} variables such as gender and race, $X$ denotes other non-protected variables used for decision making, and $Y$ is an 
\textit{outcome} random variable. 
For instance, $Y_i$ could represent a loan approval decision for individual $i$ based on demographic information $D_i$ and credit score $X_i$. We focus in this paper on discrete (or discretized) and finite domains $\calD$ and $\calX$ and binary outcomes, i.e.\ $\calY = \{0,1\}$. There is no restriction on the dimensions of $D$ and $X$. 

Our goal is to determine a randomized mapping $p_{\Xh,\Yh|X,Y,D}$ that (i) transforms the given dataset into a new dataset $\left\{(D_i,\Xh_i,\Yh_i)\right\}_{i=1}^n$, which may be used to train a model, and (ii) similarly transforms data to which the model is applied, i.e.\ test data.  Each $(\Xh_i, \Yh_i)$ is drawn independently from the same domain $\calX \times \calY$ as $X, Y$ 
by applying $p_{\Xh,\Yh|X,Y,D}$ to the corresponding triplet $(D_i,X_i,Y_i)$.  Since $D_i$ is retained as-is, we do not include it in the mapping to be determined.  Motivation for retaining $D$ is discussed later in Section \ref{sec:theory:generalize}. 
For test samples, $Y_i$ is not available at the input while $\Yh_i$ may not be needed at the output.  In this case, a reduced mapping $p_{\Xh|X,D}$ may be used, which can be obtained from $p_{\Xh,\Yh|X,Y,D}$ by marginalizing over $\Yh$ and $Y$ after weighting by $p_{Y|X,D}$.

It is assumed that 
$p_{D,X,Y}$ is known along with its marginals and conditionals. This assumption is often satisfied using the empirical distribution of $\{(D_i,X_i,Y_i)\}_{i=1}^n$.  In Section \ref{sec:theory:generalize}, we state a result ensuring that discrimination and utility loss continue to be controlled if the distribution used to determine $p_{\Xh,\Yh|X,Y,D}$ differs from the distribution of test samples.

We propose that the mapping $p_{\Xh,\Yh|X,Y,D}$ satisfy the properties discussed in the following three subsections. 

\subsection{Discrimination Control}
\label{sec:proposed:disc}
The first objective is to limit the dependence of the transformed outcome $\Yh$ on the discriminatory variables $D$, as represented by the conditional distribution $p_{\Yh|D}$. We propose two alternative formulations. 
The first requires $p_{\Yh|D}$ to be close to a target distribution $p_{Y_T}$ for all values of $D$, 
\begin{equation}\label{eqn:disc_control2}
	J\left(p_{\Yh|D}(y|d),p_{Y_T}(y)\right)\leq \epsilon_{y,d}~\forall ~d\in \calD, y\in \{0,1\},
\end{equation}  
where $J(\cdot,\cdot)$ denotes some distance function. 
The second formulation constrains $p_{\Yh|D}$ to be similar for any two values of $D$, 
\begin{equation}\label{eqn:disc_control1}
	J\left(p_{\Yh|D}(y|d_1),p_{\Yh|D}(y|d_2)\right)\leq \epsilon_{y,d_1,d_2}
\end{equation}  
for all $d_1,d_2\in \calD, y\in \{0,1\}.$ The latter \eqref{eqn:disc_control1} does not require a target distribution as reference but does increase the number of constraints from $O(\lvert\calD\rvert)$ to $O(\lvert\calD\rvert^2)$. 


The choice of target $p_{Y_T}$ in \eqref{eqn:disc_control2}, and distance $J$ and thresholds $\epsilon$ in \eqref{eqn:disc_control2} and \eqref{eqn:disc_control1} should be informed by societal considerations. 
If the application domain has a clear legal definition of disparate impact, for example the ``80\% rule'' \citep{EEOC}, then it can be translated into a mathematical constraint.  Otherwise and more generally, the instantiation of \eqref{eqn:disc_control2} should involve consultation with domain experts and stakeholders before being put into practice.  

For this work, we choose $J$ to be the following probability ratio measure:
\begin{equation}\label{eqn:probRatio}
J(p,q) = \left\lvert \frac{p}{q} - 1 \right\rvert.
\end{equation}
The combination of \eqref{eqn:probRatio} and \eqref{eqn:disc_control2} generalizes the extended lift criterion proposed in the literature \citep{pedreschi2012study}, while the combination of \eqref{eqn:probRatio} and \eqref{eqn:disc_control1} generalizes selective and contrastive lift.  In the numerical results in Section \ref{sec:results}, we use both \eqref{eqn:disc_control2} and \eqref{eqn:disc_control1}. For \eqref{eqn:disc_control2}, we make the straightforward choice of setting $p_{Y_T} = p_Y$, the original marginal distribution of the outcome variable.
We recognize however that this choice of target may run the risk of perpetuating bias in the original dataset.  On the other hand, how to choose a target distribution that is ``fairer'' than $p_Y$ is largely an open question; we refer the reader to \citet{zliobaite2011conditional} for one such proposal, which is reminiscent of the concept of ``balanced error rate'' in classification \citep{zhao2013ber}. 

In \eqref{eqn:disc_control2} and \eqref{eqn:disc_control1}, discrimination control is imposed jointly with respect to all discriminatory variables, e.g.\ all combinations of gender and race if $D$ consists of those two variables.  An alternative is to take the discriminatory variables one at a time, e.g.\ gender without regard to race and vice-versa.  The latter, which we refer to as univariate discrimination control, can be formulated similarly to \eqref{eqn:disc_control2}, \eqref{eqn:disc_control1}.  In this work, we opt for joint discrimination control as it is more stringent than univariate.  We note however that legal formulations tend to be of the univariate type.

Formulations \eqref{eqn:disc_control2} and \eqref{eqn:disc_control1} control discrimination at the level of the overall population in the dataset.  It is also possible to control discrimination within segments of the population by conditioning on additional variables $B$, where $B$ is a subset of $X$ and $X$ is a collection of features. Constraint \eqref{eqn:disc_control2} would then generalize to 
\begin{equation}\label{eqn:disc_control_cond}
	J\left(p_{\Yh|D,B}(y|d,b),p_{Y_T|B}(y|b)\right)\leq \epsilon_{y,d,b}
\end{equation}  
for all $d\in \calD,$ $ y\in \{0,1\},$ and $b \in \calB.$ Similar conditioning or ``context'' for discrimination has been explored before in \citet{hajian2013methodology} in the setting of association rule mining.  As one example, $B$ may consist of non-discriminatory variables that are strongly correlated with the outcome $Y$, e.g.\ education level as it relates to income.  One may wish 
to control for such variables in determining whether discrimination is present and needs to be corrected.  At the same time, care must be taken so that the population segments created by conditioning on $B$ are large enough for statistically valid inferences to be made. For present purposes, we simply note that conditional discrimination constraints \eqref{eqn:disc_control_cond} can be accommodated in our framework and defer further investigation to future work.
\vspace{-.1in}
\subsection{Distortion Control}
The mapping $p_{\Xh,\Yh|X,Y,D}$ should satisfy distortion constraints with respect to the domain $\calX \times \calY$. These constraints restrict the mapping to reduce or avoid altogether certain large changes (e.g. a very low credit score being mapped to a very high credit score). Given a distortion metric $\delta:(\calX\times\calY)^2\to \Reals_+$, we constrain the conditional expectation of the distortion as follows: 
\begin{equation}\label{eqn:dist_control}
    \EE{\delta\bigl((x,y), (\Xh,\Yh)\bigr) \mid D=d, X=x,Y=y} \leq c_{d,x,y}~
    \forall~(d,x,y)\in \calD \times \calX\times \calY.
\end{equation}
We assume that $\delta(x,y,x,y)=0$ for all $(x,y)\in \calX\times\calY$. 

Constraint \eqref{eqn:dist_control} is formulated with pointwise conditioning on $(D,X,Y) = (d,x,y)$ 
in order to promote \emph{individual} fairness.  It ensures that distortion is controlled for every combination of $(d,x,y)$, i.e.\ every individual in the original dataset, and more importantly, every individual to which a model is later applied.  By way of contrast, an average-case measure in which an expectation is also taken over $D, X, Y$ may result in high distortion for certain $(d,x,y)$, likely those with low probability.  Equation \eqref{eqn:dist_control} also allows the level of control $c_{d,x,y}$ to depend on 
$(d,x,y)$ if desired. We also note that \eqref{eqn:dist_control} is a property of the mapping $p_{\Xh,\Yh|D,X,Y}$, and does not depend on the assumed distribution $p_{D,X,Y}.$ 

The expectation over $\Xh, \Yh$ in \eqref{eqn:dist_control} encompasses several cases depending on the choices of the metric $\delta$ and thresholds $c_{d,x,y}$. If $c_{d,x,y} = 0$, then no mappings with nonzero distortion are allowed for individuals with original values $(d,x,y)$.  If $c_{d,x,y} > 0$, then certain mappings may still be disallowed by assigning them infinite distortion.  Mappings with finite distortion are permissible subject to the budget $c_{d,x,y}$. Lastly, if $\delta$ is binary-valued (perhaps achieved by thresholding a multi-valued distortion function), it can be seen as classifying mappings into desirable ($\delta = 0$) and undesirable ones ($\delta = 1$). Here, \eqref{eqn:dist_control} reduces to a bound on the conditional probability of an undesirable mapping, i.e.
\begin{equation}
    \label{eqn:dist_control_excess}
    \Pr\left( \delta\bigl((x,y), (\Xh,\Yh)\bigr)=1 \mid D=d, X=x,Y=y \right) \leq c_{d,x,y}.
\end{equation}

\subsection{Utility Preservation}
In addition to constraints on individual distortions, 
we also require 
that the \textit{distribution} of $(\Xh,\Yh)$ be statistically close to the distribution of $(X,Y)$. This is to ensure that a model learned from the transformed dataset (when averaged over the discriminatory variables $D$) is not too different from one learned from the original dataset, e.g.\ a bank's existing policy for approving loans. For a given dissimilarity measure $\Delta$ between probability distributions (e.g. KL-divergence), we require that $\Delta\left(p_{\Xh,\Yh},p_{X,Y}\right)$ be small.

\subsection{Optimization Formulation}
Putting together the considerations from the three previous subsections, we arrive at the optimization problem below for 
determining a randomized transformation $p_{\Xh,\Yh|X,Y,D}$ mapping each sample $(D_i, X_i, Y_i)$ to $(\Xh_i, \Yh_i)$: 
\begin{align}
    	\min_{p_{\Xh,\Yh|X,Y,D}} \;\; &\Delta\left(p_{\Xh,\Yh},p_{X,Y}\right) \nonumber\\
	    \sto \;\; &J\left(p_{\Yh|D}(y|d),p_{Y_T}(y)\right)\leq \epsilon_{y,d}~\mbox{ and} \nonumber\\
	    &\EE{\delta\bigl((x,y), (\Xh,\Yh)\bigr) \mid D=d, X=x,Y=y}~\leq c_{d,x,y} ~\forall~(d,x,y)\in \calD \times \calX\times \calY,\nonumber\\
	    &p_{\Xh,\Yh|X,Y,D} \mbox{ is a valid distribution.}\label{eqn:optimization}
\end{align}
We choose to minimize the utility loss $\Delta$ subject to constraints on individual distortion \eqref{eqn:dist_control} and discrimination, where we have used \eqref{eqn:disc_control2} for concreteness, since it is more natural to place bounds on the latter two. 

The distortion constraints \eqref{eqn:dist_control} are an essential component of the problem formulation \eqref{eqn:optimization}. Without \eqref{eqn:dist_control} and assuming that $p_{Y_T} = p_Y$, it is possible to achieve perfect utility and non-discrimination simply by sampling $(\Xh_i, \Yh_i)$ from the original distribution $p_{X,Y}$ independently of any inputs, i.e.\ $p_{\Xh,\Yh|X,Y,D}(\xh,\yh|x,y,d) = p_{\Xh,\Yh}(\xh,\yh) = p_{X,Y}(\xh,\yh)$.  Then $\Delta\left(p_{\Xh,\Yh},p_{X,Y}\right) = 0$, and $p_{\Yh|D}(y|d) = p_{\Yh}(y) = p_Y(y) = p_{Y_T}(y)$ for all $d \in \calD$. This solution however is clearly objectionable from the viewpoint of individual fairness, especially for individuals to whom a subsequent model is applied since it amounts to discarding an individual's data and replacing it with a random sample from the population $p_{X,Y}$.  Constraint \eqref{eqn:dist_control} seeks to prevent such gross deviations from occurring.

\section{Theoretical Properties}
\label{sec:theory}

\subsection{Convexity}
\label{sec:theory:convex}
We first discuss conditions under which \eqref{eqn:optimization} is a convex or quasiconvex optimization problem. Considering first the objective function, the distribution $p_{X,Y}$ is a given quantity while 
\begin{equation*}
    p_{\Xh,\Yh}(\xh,\yh) = \sum_{d,x,y} p_{D,X,Y}(d,x,y) p_{\Xh,\Yh|D,X,Y}(\xh,\yh|d,x,y)
\end{equation*}
is seen to be a linear function of the mapping $p_{\Xh,\Yh|D,X,Y}$, i.e.\ the optimization variable. Hence if the statistical dissimilarity $\Delta(\cdot,\cdot)$ is convex in its first argument with the second fixed, then $\Delta(p_{\Xh,\Yh},p_{X,Y})$ is a convex function of $p_{\Xh,\Yh|D,X,Y}$ by the affine composition property \citep{boyd2004}.  This condition is satisfied for example by all $f$-divergences \citep{csiszar_information_2004}, which are jointly convex in both arguments, and by all Bregman divergences \citep{banerjee_clustering_2005}.  If instead $\Delta(\cdot,\cdot)$ is only quasiconvex in its first argument, a similar composition property implies that $\Delta(p_{\Xh,\Yh},p_{X,Y})$ is a quasiconvex function of $p_{\Xh,\Yh|D,X,Y}$ \citep{boyd2004}. 

For discrimination constraint \eqref{eqn:disc_control2}, the target distribution $p_{Y_T}$ is assumed to be given.  The conditional distribution $p_{\Yh|D}$ can be related to $p_{\Xh,\Yh|D,X,Y}$ as follows:
\begin{align*}
    p_{\Yh|D}(\yh|d) 
    &= \sum_{\xh} \sum_{x,y} p_{X,Y|D}(x,y|d) p_{\Xh,\Yh|D,X,Y}(\xh,\yh|d,x,y).
\end{align*}
Since $p_{X,Y|D}$ is given, 
$p_{\Yh|D}$ is a linear function of $p_{\Xh,\Yh|D,X,Y}$.  Hence by the same composition property as above, \eqref{eqn:disc_control2} is a convex constraint, i.e.\ specifies a convex set, if the distance function $J(\cdot,\cdot)$ is quasiconvex in its first argument.

If constraint \eqref{eqn:disc_control1} is used instead of \eqref{eqn:disc_control2}, then both arguments of $J$ are linear functions of $p_{\Xh,\Yh|D,X,Y}$.  Hence \eqref{eqn:disc_control1} is convex if $J$ is jointly quasiconvex in both arguments.

Lastly, the distortion constraint \eqref{eqn:dist_control} can be expanded explicitly in terms of $p_{\Xh,\Yh|D,X,Y}$ to yield 
\begin{equation*}
    \sum_{\xh,\yh} p_{\Xh,\Yh|D,X,Y}(\xh,\yh|d,x,y) \delta\bigl((x,y), (\xh,\yh)\bigr) \leq c_{d,x,y}.
\end{equation*}
Thus \eqref{eqn:dist_control} is a linear constraint in $p_{\Xh,\Yh|D,X,Y}$ regardless of the choice of distortion metric $\delta$. 

We summarize this subsection with the following proposition.
\begin{prop}
Problem \eqref{eqn:optimization} is a (quasi)convex optimization if $\Delta(\cdot,\cdot)$ is (quasi)convex and $J(\cdot,\cdot)$ is quasiconvex in their respective first arguments (with the second arguments fixed). If discrimination constraint \eqref{eqn:disc_control1} is used in place of \eqref{eqn:disc_control2}, then the condition on $J$ is that it be jointly quasiconvex in both arguments.
\end{prop}

\subsection{Generalizability of Discrimination Control}
\label{sec:theory:generalize}

We now discuss the generalizability of discrimination guarantees \eqref{eqn:disc_control2} and \eqref{eqn:disc_control1} to unseen individuals, i.e.\ those to whom a model is applied.  Recall from Section \ref{sec:proposed} that the proposed transformation retains the discriminatory variables $D$.  We first consider the case where models trained on the transformed data to predict $\Yh$ are allowed to depend on $D$.  While such models may qualify as disparate treatment, the intent and effect is to better mitigate disparate impact resulting from the model. 
In this respect our proposal shares the same spirit with ``fair'' affirmative action in \citet{dwork2012fairness} (fairer on account of distortion constraint \eqref{eqn:dist_control}). Later in this subsection we consider the case where $D$ is suppressed at classification time.

\subsubsection{Maintaining the Discriminatory Variable}

Assuming that predictive models for $\Yh$ can depend on $D$, let $\Yt$ be the output of such a model based on $D$ and $\Xh$.  To remove the separate issue of model accuracy, suppose for simplicity that the model provides a good approximation to the conditional distribution of $\Yh$: $p_{\Yt|\Xh,D}(\yt|\xh,d) \approx p_{\Yh|\Xh,D}(\yt|\xh,d)$.  Then for individuals in a protected group $D=d$, the conditional distribution of $\Yt$ is given by 
\begin{align}
    p_{\Yt|D}(\yt|d) &= \sum_{\xh} p_{\Yt|\Xh,D}(\yt|\xh,d) p_{\Xh|D}(\xh|d)\nonumber\\
    &\approx \sum_{\xh} p_{\Yh|\Xh,D}(\yt|\xh,d) p_{\Xh|D}(\xh|d) = p_{\Yh|D}(\yt|d).\label{eqn:p_Yt|D1}
\end{align}
Hence the model output $p_{\Yt|D}$ can also be controlled by \eqref{eqn:disc_control2} or \eqref{eqn:disc_control1}. 

On the other hand, if $D$ must be suppressed from the transformed data, perhaps to comply with legal requirements regarding its non-use, then a predictive model can depend only on $\Xh$ and approximate $p_{\Yh|\Xh}$, i.e.\ $p_{\Yt|\Xh,D}(\yt|\xh,d) = p_{\Yt|\Xh}(\yt|\xh) \approx p_{\Yh|\Xh}(\yt|\xh)$.  In this case we have 
\begin{equation}\label{eqn:p_Yt|D2}
    p_{\Yt|D}(\yt|d) \approx \sum_{\xh} p_{\Yh|\Xh}(\yt|\xh) p_{\Xh|D}(\xh|d),
\end{equation}
which in general is not equal to $p_{\Yh|D}(\yt|d)$ in \eqref{eqn:p_Yt|D1}.  The quantity on the right-hand side of \eqref{eqn:p_Yt|D2} is less straightforward to control. We address this issue in the next subsection. 

\subsubsection{Suppressing the Discriminatory Variable}

In many applications the discriminatory variable cannot be revealed to the classification algorithm. In this case, the train-time discrimination guarantees are preserved at apply time if the Markov relationship $D\to \Xh\to \Yh$ (i.e. $p_{\Yh|\Xh,D}=p_{\Yh|\Xh}$) holds since, in this case, 
\begin{equation}\label{eqn:p_Yt|D3}
    p_{\Yt|D}(\yt|d) \approx \sum_{\xh} p_{\Yh|\Xh}(\yt|\xh) p_{\Xh|D}(\xh|d)=p_{\Yh|D}(\yt|d).
\end{equation}
Thus, given that the distribution $p_{D,X,Y}$ is known, the guarantees provided during training still hold when applied to fresh samples if the additional constraint $p_{\Xh,\Yh|D,X,Y} = p_{\Yh|\Xh}p_{\Xh|D,X,Y}$ is satisfied. We refer to \eqref{eqn:optimization} with the additional constraint $p_{\Xh,\Yh|D,X,Y} = p_{\Yh|\Xh}p_{\Xh|D,X,Y}$ as the \textit{suppressed optimization formulation} (SOF). Alas, since the added constraint is non-convex, the SOF is not a convex program, despite being convex in $p_{\Xh|D,X,Y}$ for a fixed $p_{\Yh|\Xh}$ and vice-versa (i.e. it is biconvex). We propose next two strategies for addressing this problem.

\begin{enumerate}
\item The first approach is to restrict $p_{\Yh|\Xh} = p_{Y|X}$ and solve \eqref{eqn:optimization} for $p_{\Xh|D,X,Y}$. If $\Delta(\cdot,\cdot)$ is an $f$-divergence, then
\begin{align*}
\Delta\left(p_{X,Y},p_{\Xh,\Yh}\right) &= D_f\left( p_{X,Y}\| p_{\Xh,\Yh} \right)\\
& = \sum_{x,y} p_{\Xh,\Yh}(x,y)f\left(\frac{p_{X,Y}(x,y)}{p_{\Xh,\Yh}(x,y)} \right) \\
&\geq \sum_x p_{\Xh}(x)f\left(\sum_y p_{\Yh|\Xh}(x|y) \frac{p_{X,Y}(x,y)}{p_{\Xh,\Yh}(x,y)} \right) \\
& = D_f\left(p_X\|p_\Xh \right),
\end{align*}
where the inequality follows from convexity of $f$. Since the last quantity is achieved by setting $p_{\Yh|\Xh} = p_{Y|X}$, this choice is optimal in terms of the objective function. 
It may, however, render the constraints in \eqref{eqn:optimization} infeasible. Assuming feasibility is maintained, this approach has the added benefit that a classifier $f_\theta(x)\approx p_{Y|X}(\cdot|x)$ can be trained using the original (non-perturbed) data, and maintained for classification during apply time.

\item Alternatively, a solution can be found through alternating minimization: fix $p_{\Yh|\Xh}$ and solve the SOF for  $p_{\Xh|D,X,Y}$, and then fix $p_{\Xh|D,X,Y}$ as the  optimal solution and solve the SOF for $p_{\Yh|\Xh}$. The resulting sequence of values of the objective function is non-increasing, but may converge to a local minima.
\end{enumerate}

\subsection{A Note on Estimation and Discrimination}

There is a close relationship between estimation and discrimination. If the discriminatory variable $D$ can be reliably estimated from the outcome variable $Y$, then it is reasonable to expect that the discrimination control constraint \eqref{eqn:disc_control2} does not hold for small values of $\epsilon_{y,d}$. We make this intuition precise in the next proposition when $J$ is given in \eqref{eqn:probRatio}. 

More specifically, we prove that if the advantage of estimating $D$ from $Y$ over a random guess is large, then there must exist a value of $d$ and $y$ such that $J(p_{Y|D}(y|d),p_{Y_T}(y))$ is also large. Thus, standard estimation methods can be used to detect the presence of discrimination: if an estimation algorithm can estimate $D$ from $Y$, then discrimination may be present. Alternatively, if discrimination control is successful, then no estimator can significantly improve upon a random guess when estimating $D$ from $Y$.

We denote the highest probability of correctly guessing $D$ from an observation of $Y$ by $P_c(D|Y)$, where
\begin{equation}
P_c(D|Y)\defined \max_{D\to Y\to\hat{D}} \Pr\left(D=\hat{D}\right),
\end{equation}
and the maximum is taken across all estimators $p_{\hat{D}|Y}$ that satisfy the Markov condition $D\to Y \to \hat{D}$. For $D$ and $Y$ defined over finite supports, this is achieved by the maximum \textit{a posteriori} (MAP) estimator and, consequently,
\begin{equation}
P_c(D|Y) = \sum_{y\in \calY} p_Y(y) \max_{d\in \calD} p_{D|Y}(d|y).
\end{equation}
Let $p_D^*$ be the most likely outcome of $D$, i.e. $p_D^*\defined \max_{d\in \calD} p_D(d)$. The (multiplicative) advantage over a random guess is given by
\begin{equation}
\mathsf{Adv}(D|Y)\defined \frac{P_c(D|Y)}{p_D^*}.
\end{equation}

\begin{prop}
For $D$ and $Y$ defined over finite support sets, if
\begin{equation}
\mathsf{Adv}(D|Y) > 1+\epsilon
\end{equation}
then for any $p_{Y_T}$, there exists $y\in \calY$ and $d\in \calD$ such that
\begin{equation}
\left| \frac{p_{Y|D}(y|d)}{p_{Y_T}(y)}-1 \right|> \epsilon.
\end{equation}
\end{prop}
\begin{proof}
We prove the contrapositive of the statement of the proposition. Assume that
\begin{equation}
\label{eq:assumption_J}
    \left| \frac{p_{Y|D}(y|d)}{p_{Y_T}(y)}-1 \right|\leq \epsilon~\forall y\in \calY,d\in \calD.
\end{equation}
Then
\begin{align*}
    P_c(D|Y) & = \sum_{y\in \calY} \max_{d\in \calD} p_{D|Y}(d|y)p_Y(y)\\
             & = \sum_{y\in \calY} \max_{d\in \calD} p_{Y|D}(y|d)p_D(d) \\
             & \leq \sum_{y\in \calY}  \max_{d\in\calD} (1+\epsilon)p_{Y_T}(y)p_D(d)\\
             & = (1+\epsilon) \max_{d\in \calD} p_D(d),
\end{align*}
where the inequality follows by noting that \eqref{eq:assumption_J} implies $p_{Y|D}(y|d)\leq (1+\epsilon)p_{Y_T}(y)$ for all $y\in \calY$, $d\in \calD$. Rearranging the terms of the last equality, we arrive at
\begin{equation*}
   \frac{ P_c(D|Y)}{\max_{d\in \calD} p_D(d)} \leq 1+\epsilon,
\end{equation*}
and the result follows by observing that the left-hand side is the definition of $\mathsf{Adv}(D|Y)$.
\end{proof}

\subsection{Training and Application Considerations}

The proposed optimization framework has two modes of operation (Fig. \ref{fig:disc_transform_pipeline}): train and apply. In train mode, the optimization problem \eqref{eqn:optimization} is solved in order to determine a mapping $p_{\Xh,\Yh|X,Y,D}$ for randomizing the training set. The randomized training set, in turn, is used to fit a classification model $f_\theta(\Xh,D)$ that approximates $p_{\Yh|\Xh,D}$, where $\theta$ are the parameters of the model. At apply time, a new data point $(X,D)$ is received and transformed into $(\Xh,D)$ through a randomized mapping $p_{\Xh|X,D}$. The mapping $p_{\Xh|D,X}$ is given by marginalizing over $Y,\Yh$:
\begin{equation}
    p_{\Xh|D,X}(\hat{x}|d,x) = \sum_{y,\hat{y}} p_{\Xh,\Yh|X,Y,D}(\hat{x},\hat{y} | x,y,d)p_{Y|X,D}(y|x,d).
\end{equation}

Assuming that the variable $D$ is not suppressed, and that the marginals are known, then the utility and discrimination guarantees set during train time still hold during apply time, as discussed in Section \ref{sec:theory:generalize}. However, the distortion control will inevitably change, since the mapping has been marginalized over $Y$. More specifically, the bound on the expected distortion for each sample becomes
\begin{equation}
\label{eq:applyguarantee}
    \EE{ \EE{\delta\bigl((x,Y), (\Xh,\Yh)\bigr) \mid D=d, X=x,Y} \mid D=d, X=x} \leq \sum_{y\in \calY} p_{Y|X,D}(y|x,d) c_{x,y,d}\defined c_{x,d}~.
\end{equation}
If the distortion control values $c_{x,y,d}$ are independent of $y$, then the upper-bound on distortion set during training time still holds during apply time. Otherwise, \eqref{eq:applyguarantee} provides a bound on individual distortion at apply time. The same guarantee holds for the case when $D$ is suppressed.

\subsection{Robustness}

We may also consider the case where the distribution $p_{D,X,Y}$ used to determine the transformation differs from the distribution $q_{D,X,Y}$ of test samples. This occurs, for example, when $p_{D,X,Y}$ is the empirical distribution computed from $n$ i.i.d.\ samples from an unknown distribution $q_{D,X,Y}$.  In this situation, discrimination control and utility are still guaranteed for samples drawn from $q_{D,X,Y}$ that are transformed using $p_{\Yh,\Xh|X,Y,D}$, where the latter is obtained by solving \eqref{eqn:optimization} with $p_{D,X,Y}$. In particular, denoting by $q_{\Yh|D}$ and $q_{\Xh,\Yh}$ the corresponding distributions 
for $\Yh,\Xh$ and $D$ when $q_{D,X,Y}$ is transformed using $p_{\Yh,\Xh|X,Y,D}$, we have  $J\left(p_{\Yh|D}(y|d), p_{Y_T}(y)\right)\to J\left(q_{\Yh|D}(y|d), p_{Y_T}(y)\right)$ and $\Delta\left(p_{X,Y},p_{\Xh,\Yh}\right)\to \Delta\left(q_{X,Y},q_{\Xh,\Yh}\right) $ for $n$ sufficiently large (the distortion control constraints \eqref{eqn:dist_control} only depend on $p_{\Yh,\Xh|X,Y,D}$). The next proposition provides an estimate of the rate of this convergence in terms of $n$ and assuming $p_{Y,D}(y,d)$ is fixed and bounded away from zero. Its proof can be found in the Appendix.

\begin{prop}
\label{prop:robust}
Let $p_{D,X,Y}$ be the empirical distribution obtained from $n$ i.i.d.\ samples that is used to determine the mapping $p_{\Yh,\Xh|X,Y,D}$, and $q_{D,X,Y}$ be the true distribution of the data. In addition, denote by $q_{D,\Xh,\Yh}$ the joint 
distribution after applying $p_{\Yh,\Xh|X,Y,D}$ to samples from $q_{D,X,Y}$. If for all $y\in \calY$, $d\in \calD$ we have $p_{Y,D}(y,d)>0$, $J\left(p_{\Yh|D}(y|d), p_{Y_T}(y)\right)\leq \epsilon$, where $J$ is given in \eqref{eqn:probRatio},
and
\begin{equation*}
\Delta\left(p_{X,Y},p_{\Xh,\Yh}\right)=\sum_{x,y} \left| p_{X,Y}(x,y)- p_{\Xh,\Yh}(x,y)\right|\leq \mu,
\end{equation*}
then with probability $1-\beta$,
\begin{align}
J\left(q_{\Yh|D}(y|d), p_{Y_T}(y)\right) &= \epsilon + O\left(\sqrt{\frac{1}{n}\log\frac{n}{\beta}}\right),\\
\Delta\left(q_{X,Y},q_{\Xh,\Yh}\right)&= \mu + O\left(\sqrt{\frac{1}{n}\log\frac{n}{\beta}}\right).
\end{align}
\end{prop}


Proposition \ref{prop:robust} guarantees that, as long as $n$ is sufficiently large, the utility and discrimination control guarantees will approximately hold when $p_{\Xh,\Yh|Y,X,D}$ is applied to fresh samples drawn from $q_{D,X,Y}$. In particular, the utility and discrimination guarantees will converge to the ones used as parameters in the optimization at a rate that is at least $\Theta\left(\sqrt{\frac{1}{n}\log n} \right)$. The distortion control guarantees \eqref{eqn:dist_control} are a property of the mapping $p_{\Xh,\Yh|Y,X,D}$, and do not depend on the distribution of the data.

Observe that hidden within the big-O terms in Proposition \ref{prop:robust} are constants that depend on the probability of the least likely symbol and the alphabet size. The exact characterization of these constants can be found in the proof of the proposition in the appendix. Moreover, the upper bounds become loose if $p_{Y,D}(y,d)$ can be made arbitrarily small. Thus, it is necessary to assume that $p_{Y,D}(y,d)$ is fixed and bounded away from zero. Moreover, if the dimensionality of the support sets of $D,X$ and $Y$ is large, and the number of samples $n$ is limited, then a dimensionality reduction step (e.g. clustering) may be necessary in order to assure that discrimination control and utility are adequately preserved at test time. Proposition \ref{prop:robust} and its proof can be used to provide an explicit estimate of the required reduction. Finally, we also note that if there are insufficient samples to reliably estimate $q_{D,X,Y}(d,x,y)$ for certain values $(d,x,y)\in \calD\times\calX\times\calY$, then, for those groups $(d,x)$, it is statistically challenging to verify discrimination and thus control may not be meaningful. 

\section{Applications to Datasets}
\label{sec:results}

We apply our proposed data transformation approach to two different datasets to demonstrate its capabilities. We approximate $p_{D,X,Y}$ using the empirical distribution of $(D,X,Y)$ in the datasets,  specialize the optimization \eqref{eqn:optimization} according to the needs of the application, and solve \eqref{eqn:optimization} using a standard convex solver \citep{cvxpy}.

\subsection{ProPublica's COMPAS Recidivism Data}
Recidivism refers to a person's relapse into criminal behavior. It has been found that about two-thirds of prisoners in the US are re-arrested after release \citep{durose2014recidivism}. It is important therefore to understand the recidivistic tendencies of incarcerated individuals who are considered for release 
at several points in the criminal justice system (bail hearings, parole, etc.). 
Automated risk scoring mechanisms have been developed for this purpose and are currently used in courtrooms in the US, in particular the proprietary COMPAS tool by Northpointe \citep{compas_tool}. 

Recently, ProPublica published an article that investigates racial bias in the COMPAS algorithm \citep{propublica_article}, releasing an accompanying dataset that includes COMPAS risk scores, recidivism records, and other relevant attributes \citep{compas_risk_scores}. A basic finding is that the COMPAS algorithm tends to assign higher scores 
to African-American individuals, a reflection of the \textit{a priori} higher prevalence of recidivism in this group.  The article goes on to demonstrate unequal false positive and false negative rates between African-Americans and Caucasian-Americans, which has since been shown by \citet{chouldechova2016fair} to be a necessary consequence of the calibration of the model and the difference in a priori prevalence. 

In this work, our interest is not in the debate surrounding the COMPAS algorithm but rather in the underlying recidivism data \citep{compas_risk_scores}. Using the proposed data transformation approach, we demonstrate the technical feasibility of mitigating the disparate impact of recividism records on different demographic groups while also preserving utility and individual fairness. (We make no comment on the associated societal considerations.)
From ProPublica's dataset, we select severity of charge, number of prior crimes, and age category to be the decision variables ($X$). The outcome variable ($Y$) is a binary indicator of whether the individual recidivated (re-offended), and race and gender are set to be the discriminatory variables ($D$). The encoding of the 
decision and discrimination variables is described in Table \ref{tab:propublica_values}. The dataset was processed to contain around 
5k records.

\begin{table}[tb]
    \centering
    \caption{\small{ProPublica dataset features.}}
    {\footnotesize
    \resizebox{\linewidth}{!}{
    \begin{tabular}{|c|c|c|}
    \hline
        Feature & Values & Comments \\
        \hline
         Recidivism (binary) & $\{0, 1\}$ & 1 if re-offended, 0 otherwise\\ 
         Gender & \{Male, Female\} &  \\
         Race & \{Caucasian, African-American\} & Races with small samples removed\\
         Age category & $\{< 25, 25-45, > 45\}$& years of age \\
         Charge degree & \{Felony, Misdemeanor\} &  For the current arrest\\
         Prior counts & $\{0, 1-3, > 3\}$ & Number of prior crimes \\
         \hline
    \end{tabular}
    }
    }
    \label{tab:propublica_values}
\end{table}

\textbf{Specific Form of Optimization.}  We specialize our general formulation in \eqref{eqn:optimization} by setting the utility measure $\Delta(p_{X,Y}, p_{\Xh,\Yh})$ to be the KL divergence $\DKL(p_{X,Y}\| p_{\Xh,\Yh})$.
For discrimination control, we use \eqref{eqn:disc_control1}, with $J$ given in \eqref{eqn:probRatio},
while fixing $\epsilon_{y,d_1,d_2} = \epsilon$. For the sake of simplicity, we use the expected distortion constraint in \eqref{eqn:dist_control}  with $c_{d,x,y} = c$ uniformly.
The distortion function $\delta$ in \eqref{eqn:dist_control} has the following form. Jumps of more than one category in age and prior counts are heavily discouraged by setting a high distortion penalty ($10^4$) for such transformations. We impose the same penalty on increases in recidivism (change of $Y$ from $0$ to $1$). Both these choices are made to promote individual fairness. Furthermore, for every jump to the next category for age and prior counts, a penalty of $1$ is assessed, and a similar jump incurs a penalty of $2$ for charge degree. Reduction in recidivism ($1$ to $0$) has a penalty of $2$. The total distortion for each individual is the sum of squares of distortions for each attribute of $X$. These distortion values were chosen for demonstration purposes to be reasonable to our judgement, and can easily be tuned according to the needs of a practitioner.



\begin{figure*}[!tb]
    \centering
    \includegraphics[scale=.4]{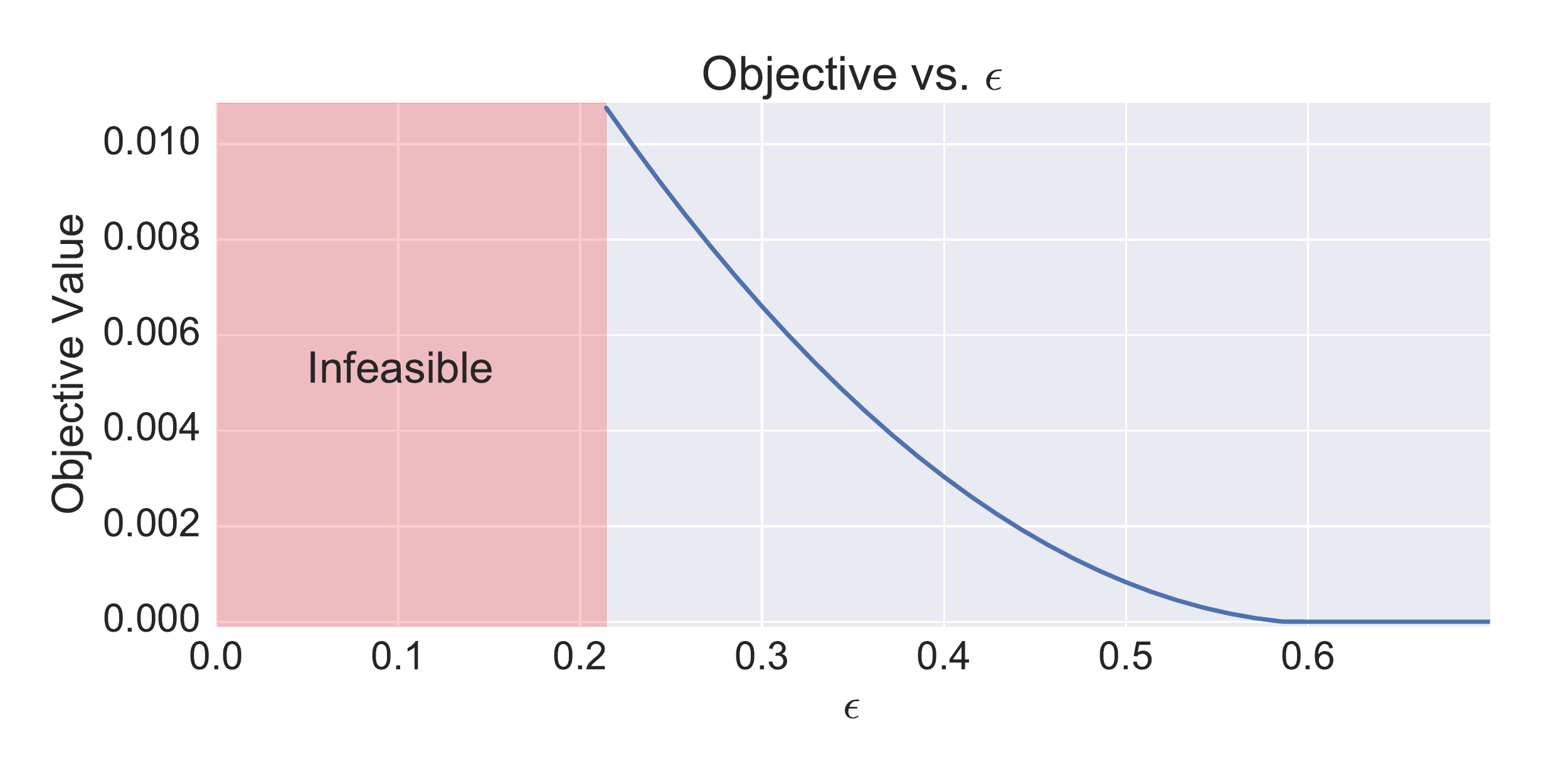}
    \caption{\small{Objective vs. discrimination parameter $\epsilon$ for  distortion constraint $c = 0.25$.}}
    \label{fig:tradeoff_compas}
\end{figure*}

\begin{figure*}[!tb]
    \centering
    \includegraphics[scale=.38]{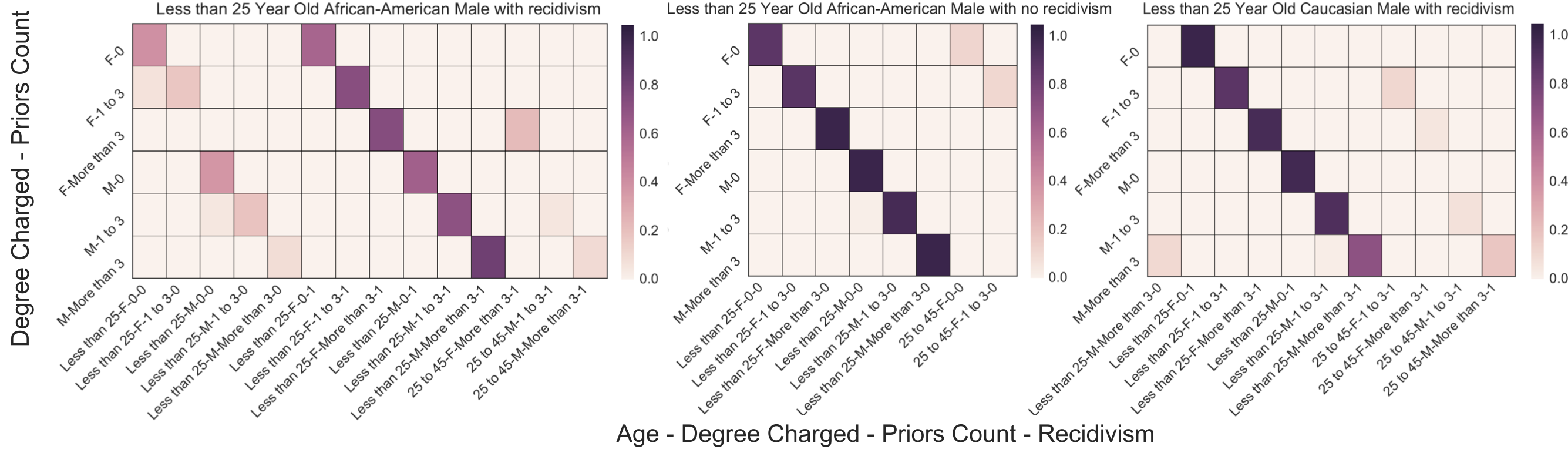}
    \vspace{-0.2in}
    \caption{\small{Conditional mappings $p_{\Xh,\Yh|X,Y,D}$ with $\epsilon = 0.1$, and $c = 0.5$ for: (\textbf{left}) $D = (\text{African-American}, \text{Male})$, less than $25$ years ($X$), $Y = 1$, (\textbf{middle}) $D = (\text{African-American}, \text{Male})$,  less than $25$ years ($X$), $Y = 0$, and (\textbf{right}) $D = (\text{Caucasian}, \text{Male})$,  less than $25$ years ($X$), $Y = 1$. Original charge degree and prior counts ($X$) are shown in vertical axis, while the transformed age category, charge degree, prior counts and recidivism $(\Xh, \Yh)$ are represented along the horizontal axis. The charge degree F indicates felony and M indicates misdemeanor. Colors indicate mapping probability values. Columns included only if the sum of its values exceeds $0.05$.}
    }
    \label{fig:cond_mappings_compas}
\end{figure*}

\begin{figure*}[!tb]
    \centering
    \includegraphics[scale=.17]{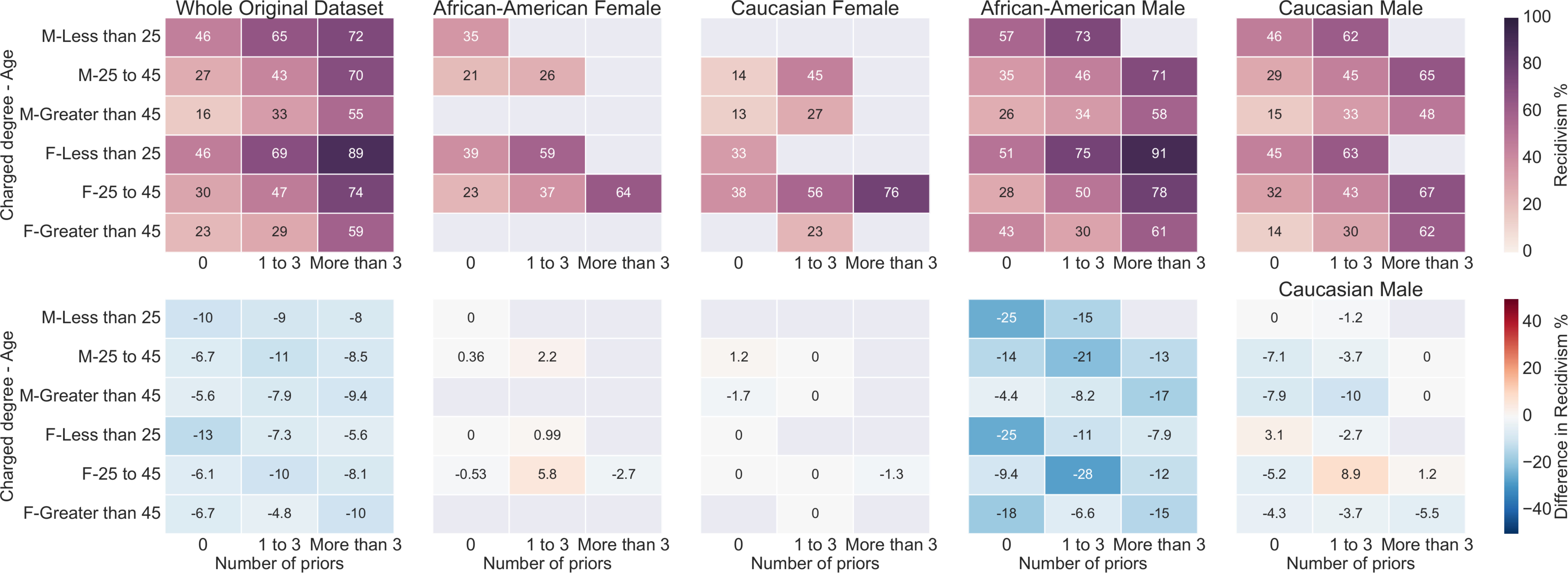}
    \caption{\small{Top row: Percentage recidivism rates in the original dataset as a function of charge degree, age and prior counts for the overall population (i.e.\ $p_{Y|X}(1|x)$) and for different groups ($p_{Y|X,D}(1|x,d)$). Bottom row: Change in percentages due to transformation, i.e.\ $p_{\Yh|\Xh,D}(1|x,d) - p_{Y|X,D}(1|x,d)$, etc. Values for cohorts of charge degree, age, and prior counts with fewer than $20$ samples are not shown. The discrimination and distortion constraints are set to $\epsilon = 0.1$ and $c = 0.5$ respectively.}
    }
    \label{fig:diffplot_compas}
\end{figure*}

\textbf{Results.} We computed the optimal objective value (i.e., KL divergence) resulting from solving \eqref{eqn:optimization} for different values of the discrimination control parameter $\epsilon$, when the expected distortion constraint $c = 0.25$. Around $\epsilon = 0.2$, no feasible solution can be found that also satisfies the distortion constraint. Above $\epsilon = 0.59$, the discrimination control is loose enough to be satisfied by the original dataset with just an identity mapping ($\DKL(p_{X,Y}\| p_{\Xh,\Yh}) = 0$). In between, the optimal value varies as a smooth function (Fig. \ref{fig:tradeoff_compas}).

We set $c = 0.5$ and $\epsilon=0.1$ for the rest of the experiments. The optimal value of utility measure (KL divergence) was $0.021$. In order to evaluate if discrimination control was achieved as expected, we examine the dependence of the outcome variable on the discrimination variable before and after the transformation. Note that to have zero disparate impact, we would like the $\Yh$ to be independent of $D$, but practically it will be controlled by the discrimination control parameter $\epsilon$. 
The corresponding marginals $p_{Y|D}$ and $p_{\Yh|D}$ are illustrated in Table \ref{tab:recidivism_rates}, where clearly $\Yh$ is less dependent on $D$ compared to $Y$. In particular, since an increase in recidivism is heavily penalized, the net effect of the randomized transformation is to decrease the recidivism risk of males, and particularly African-American males.

The mapping $p_{\Xh, \Yh|X,Y,D}$ produced by the optimization \eqref{eqn:optimization} can reveal important insights on the nature of disparate impact and how to mitigate it. We illustrate this by exploring  $p_{\Xh, \Yh|X,Y,D}$  for the COMPAS dataset next. Fig.~\ref{fig:cond_mappings_compas} displays the conditional mapping restricted to certain socio-demographic groups. First consider young males who are African-American (left-most plot). This group has a high recidivism rate, and hence the most prominent action of the mapping (besides identity transformation) is to change the recidivism value from $1$ (recidivism) to $0$ (no recidivism). The next prominent action is to change the age category from young to middle aged (25 to 45 years). This effectively reduces the average value of $\Yh$ for young African-Americans, since the mapping for young males who are African-American and do not recidivate (middle plot) is essentially the identity mapping, with the exception of changing age category to middle aged. This is expected, since increasing recidivism is heavily penalized. For young Caucasian males who recidivate, the action of the proposed transformation seems to be similar to that of young African-American males who recidivate, i.e., the outcome variable is either changed to 0, or the age category is changed to middle age. However the probabilities of the transformations are lower since Caucasian males have, according to the dataset, a lower recidivism rate.


We apply this conditional mapping on the dataset (one trial) and present the results in Fig.~\ref{fig:diffplot_compas}. The original percentage recidivism rates are also shown in the top panel of the plot for comparison. Because of our constraint that disallows changing the outcome to $1$, a demographic group's recidivism rate can  (indirectly) increase only  through changes to the decision variables ($X$). We note that the average percentage change in recidivism rates across all demographics is negative when the discrimination variables are marginalized out (leftmost column). The maximum decreases in recidivism rates are observed for African-American males since they have the highest value of $p_{Y|D}(1|d)$ (cf. Table~\ref{tab:recidivism_rates}). Contrast this with Caucasian females (middle column), who have virtually no change in their recidivism rates since they are a priori close to the final ones (see Table~\ref{tab:recidivism_rates}). Another interesting observation is that middle aged Caucasian males with 1 to 3 prior counts see an increase in percentage recidivism. This is consistent with the mapping seen in Fig.~\ref{fig:cond_mappings_compas} (middle), and is an example of the indirect introduction of positive outcome variables in a cohort as discussed above. 


\begin{table}
\scriptsize
\centering
{\def\arraystretch{1.3}\tabcolsep=2pt

  \centering
  \caption{\small{Dependence of the outcome variable on the discrimination variable before and after the proposed transformation. F and M indicate Female and Male, and A-A, and C indicate African-American and Caucasian.}}
  {
    \begin{tabular}{| c |c|c|c|c|}
    \hline
    D & \multicolumn{2}{|c|}{Before transformation} & \multicolumn{2}{|c|}{After transformation} \\

\cline{2-5}   (gender,~race) &$p_{Y|D}(0|d)$ & $p_{Y|D}(1|d)$ & $p_{\Yh|D}(0|d)$ & $p_{\Yh|D}(1|d)$ \\
    \hline
    F, A-A & 0.607 & 0.393 & 0.607 & 0.393 \\
\cline{1-1}    F, C & 0.633 & 0.367 & 0.633 & 0.367 \\
\cline{1-1}   M, A-A & 0.407 & 0.593 & 0.596 & 0.404 \\
\cline{1-1}    M, C & 0.570 & 0.430 & 0.596 & 0.404 \\
    \hline
    \end{tabular}}
      \label{tab:recidivism_rates}
 }
\end{table}

\begin{figure*}[!tb]
    \centering
    \includegraphics[scale = .17]{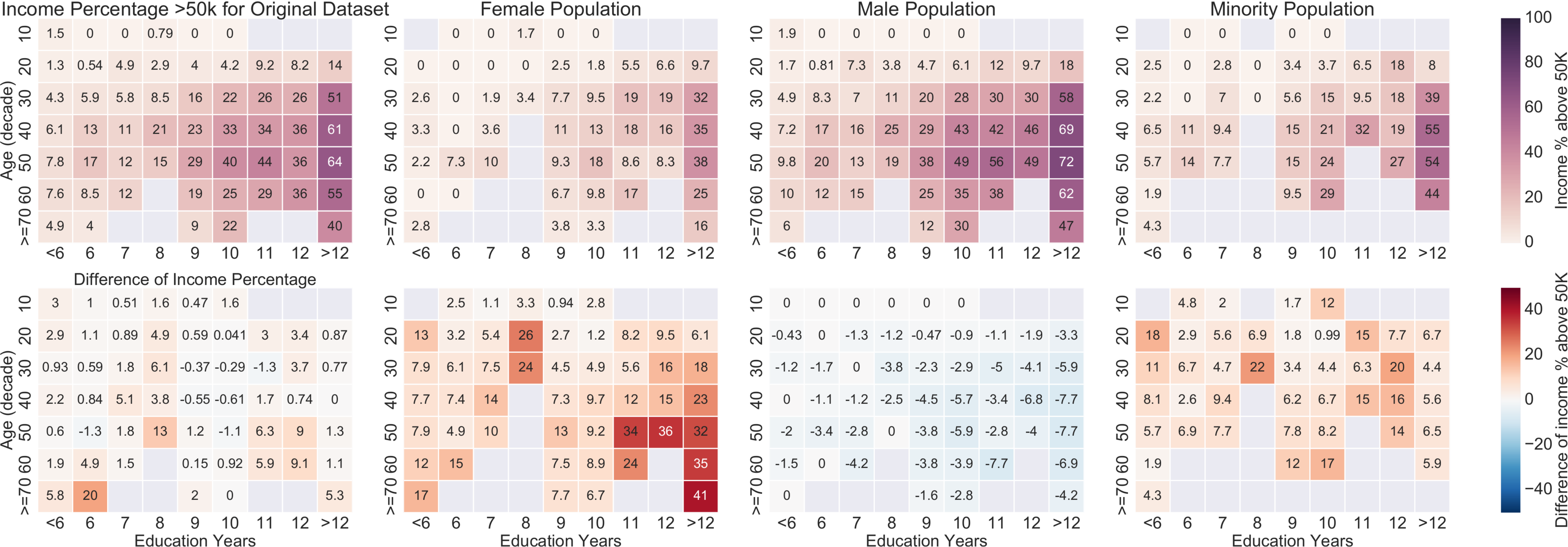}
    \caption{\small{Top row: High income percentages in the original dataset as a function of age and education for the overall population (i.e.\ $p_{Y|X}(1|x)$) and for different groups $p_{Y|X,D}(1|x,d)$). Bottom row: Change in percentages due to transformation, i.e.\ $p_{\Yh|\Xh,D}(1|x,d) - p_{Y|X}(1|x,d)$, etc. Age-education pairs with fewer than 20 samples are not shown.}}
    \label{fig:income_distr}
\end{figure*}

\subsection{UCI Adult Data}
\label{sec:uci_adult}

We apply our optimization approach to the well-known UCI Adult Dataset \citep{Lichman:2013} as a second illustration of its capabilities. 
The features were categorized as discriminatory variables ($D$): Race (White, Minority) and Gender (Male, Female); decision variables ($X$): Age (quantized to decades) and Education (quantized to years); and response variable ($Y$): Income (binary). While the response variable considered here is income, the dataset could be regarded as a simplified proxy for analyzing other financial outcomes such as credit approvals.

\textbf{Specific Form of Optimization.} 
We use $\ell_1$-distance (twice the total variation) \citep{Pollard2002} to measure utility, $\Delta\left(p_{X,Y},p_{\Xh,\Yh}\right) =\sum_{x,y} \left| p_{X,Y}(x,y)- p_{\Xh,\Yh}(x,y)\right|$.
For discrimination control, we use \eqref{eqn:disc_control2},  with $J$ given in \eqref{eqn:probRatio},
We also set $\epsilon_{y,d} = \epsilon$ in \eqref{eqn:disc_control2}. We use the distortion function in \eqref{eqn:dist_control}, and write $x = (a,e)$ for an age-education pair and $\xh = (\ah, \eh)$ for a corresponding transformed pair.  The distortion function returns (i) $v_1$ if income is decreased, age is not changed and education is increased by at most 1 year, (ii) $v_2$ if age is changed by a decade and education is increased by at most 1 year regardless of the change of income, (iii) $v_3$ if age is changed by more than a decade or education is lowered by any amount or increased by more than 1 year, and (iv) 0 in all other cases.  We set $(v_1, v_2, v_3) = (1, 2, 3)$ with corresponding distance thresholds for $\delta = 0$ as $(0.9, 1.9, 2.9)$ and corresponding probabilities ($c_{d,x,y}$) as $(0.1, 0.05, 0)$ in \eqref{eqn:dist_control}.  As a consequence, decreases in income, small changes in age, and small increases in education (events (i), (ii)) are permitted with small probabilities, while larger changes in age and education (event (iii)) are not allowed at all. We note that the parameter settings are selected with the purpose of demonstrating our approach, and would change depending on the practitioner's requirements or guidelines.

\textbf{Results.} For the remainder of the results presented here, we set $\epsilon=0.15$, and the optimal value of the utility measure ($\ell_1$ distance) was $0.014$. We apply the conditional mapping, generated as the optimal solution to \eqref{eqn:optimization}, to transform the age, education, and income values of each sample in the dataset. The result of a single realization of this randomization is given in Fig.~\ref{fig:income_distr}, where we show percentages of high income individuals as a function of age and education before and after the transformation. The original age and education ($X$) are plotted throughout Fig.~\ref{fig:income_distr} for ease of comparison, and that changes in individual percentages may be larger than a factor of $1\pm\epsilon$ because discrimination is not controlled by \eqref{eqn:disc_control2} at the level of age-education cohorts. The top left panel indicates that income is higher for more educated and middle-aged people, as expected. The second column shows that high income percentages are significantly lower for females and are accordingly increased by the transformation, most strongly for educated older women and younger women with only 8 years of education, and less so for other younger women. Conversely, the percentages are decreased for males but by much smaller magnitudes. Minorities receive small percentage increases but less than for women, in part because they are a more heterogeneous group consisting of both genders.

\section{Conclusions}
We proposed a flexible, data-driven optimization framework for probabilistically transforming data in order to reduce algorithmic discrimination, and applied it to two datasets. The differences between the original and transformed datasets revealed interesting discrimination patterns, as well as corrective adjustments for controlling discrimination while preserving utility of the data. Despite being programmatically generated, the optimized transformation satisfied properties that are sensible from a socio-demographic standpoint, reducing, for example, recidivism risk for males who are African-American in the recidivism dataset, and increasing income for well-educated females in the UCI adult dataset. The flexibility of the approach allows numerous extensions using different measures and constraints for utility preservation, discrimination, and individual distortion control. Investigating such extensions, 
developing theoretical characterizations based on the proposed framework, and quantifying the impact of the transformations on specific supervised learning tasks will be pursued in future work.

\appendices

\section{Proof of Proposition \ref{prop:robust}}
The proposition is a consequence of the following elementary lemma.

\begin{lem}
\label{lem:keyLemma}
Let $p(x)$, $q(x)$ and $r(x)$ be three fixed probability mass functions with the same discrete and finite support set $\calX$, $c_1 \defined \min_{x\in\calX} \frac{p(x)(1-p(x))}{3(1+p(x))^2}>0$ and $p_m \defined \min_x p(x)>0$. Then if 
\begin{equation}
    \label{eq:KLassump}
    \DKL\left(p\|q\right) \leq \tau \leq c_1
\end{equation}
and for all $x\in \calX$ and
\begin{equation}
    \gamma_1 \leq \frac{p(x)}{r(x)} \leq \gamma_2,
\end{equation}
then for all $x\in \calX$ and $g(\tau,p_m)\defined \sqrt{\frac{3\tau}{p_m}}$
\begin{equation}
    \gamma_1 \exp\left(-g(\tau,p_m) \right)\leq \frac{q(x)}{r(x)} \leq \gamma_2 \exp\left(g(\tau,p_m) \right).
\end{equation}
\end{lem}
\begin{proof}
We assume $\tau>0$, otherwise $p(x)=q(x)$ $\forall x\in \calX$ and we are done. From \eqref{eq:KLassump} and the Data Processing Inequality for KL-divergence, for any $x\in \cal X$
\begin{align}
    \label{eq:step1}
    p(x)\log \frac{p(x)}{q(x)} + (1-p(x))\log \frac{1-p(x)}{1-q(x)}\leq \tau.
\end{align}
Let $x$ be fixed, and, in order to simplify notation, denote $c\defined p(x)$. Assuming, without loss of generality, $$q(x)=c\exp\left(-\frac{\alpha\tau}{c} \right),$$ then \eqref{eq:step1} implies
\begin{align}
\label{eq:step2}
f(\alpha)\defined \alpha - \frac{1-c}{\tau}\log\left( \frac{1-c\exp\left(-\frac{\alpha\tau}{c} \right)}{1-c}\right)\leq 1.
\end{align}
The Taylor series of $f(\alpha)$ around 0 has the form
\begin{equation}
\label{eq:taylor}
f(\alpha) = \sum_{n=2}^\infty \frac{(-1)^n}{n!}\left(\frac{\tau}{(1-c)c}\right)^{n-1}A_{n-1}(c)\alpha^n,
\end{equation}
where $A_n(c)$ is the Eulerian polynomial, which is positive for $c>0$ and satisfies $A_1(c)=1$ and $A_2(c)=(1+c)$. First, assume $\alpha \leq 0$. Then $f(\alpha)$ can be lower-bounded by the first term in its Taylor series expansion since all the terms in the series are non-negative. From \eqref{eq:step2}, 
\begin{equation}
\frac{\tau\alpha^2}{2(1-c)c}\leq f(\alpha)\leq 1.
\end{equation}
Consequently, 
\begin{equation}
\label{eq:firstBound}
\alpha \geq -\sqrt{\frac{2(1-c)c}{\tau}}.
\end{equation}
Now assume $\alpha \geq 0$. Then the Taylor series \eqref{eq:taylor} becomes an alternating series, and $f(\alpha)$ can be lower-bounded by its first two terms
\begin{equation}
\frac{\tau\alpha^2}{2(1-c)c}-\frac{(1+c)\tau^2\alpha^3}{6(1-c)^2c^2}\leq f(\alpha)\leq 1.
\end{equation}
The term in the l.h.s. of the first inequality satisfies
\begin{equation}
\label{eq:step3}
\frac{\tau\alpha^2}{3(1-c)c}\leq\frac{\tau\alpha^2}{2(1-c)c}-\frac{(1+c)\tau^2\alpha^3}{6(1-c)^2c^2}
\end{equation}
as long as $\alpha\leq \frac{c(1-c)}{(1+c)\tau}$. Since the  lhs is larger than 1 when $\alpha > \sqrt{\frac{3(1-c)c}{\tau}},$ then it is a valid lower-bound for $f(\alpha)$ in the entire interval where $f(\alpha)\leq 1$ and $\alpha\geq 0$ as long as
\begin{equation}
\sqrt{\frac{3(1-c)c}{\tau}}\leq \frac{c(1-c)}{(1+c)\tau} \Leftrightarrow \tau\leq \frac{c(1-c)}{3(1+c)^2},
\end{equation}
which holds by assumption in the Lemma. Thus, 
\begin{equation}
\alpha \leq \sqrt{\frac{3(1-c)c}{\tau}},
\end{equation}
and combining the previous equation with \eqref{eq:firstBound}
\begin{equation}
-\sqrt{\frac{2(1-c)c}{\tau}}\leq \alpha \leq \sqrt{\frac{3(1-c)c}{\tau}}
\end{equation}
Finally, since $\frac{q(x)}{p(x)} =\exp(-\alpha \tau/p(x))$, from the previous inequalities
\begin{align}
    \exp \left(-\sqrt{\frac{3(1-p(x))\tau}{p(x)}} \right)&\leq \frac{q(x)}{p(x)}\nonumber\\
    &\leq \exp \left(\sqrt{\frac{2(1-p(x))\tau}{p(x)}} \right),\label{eq:finalkey}
\end{align}
and the result follows by further lower bounding the lhs by $\gamma_1 r(x)\leq p(x)$ and upper bounding the rhs by $ p(x)\geq \gamma_2 r(x)$
\end{proof}

The previous Lemma allows us to derive the result presented in Proposition \ref{prop:robust}. 

\begin{proof}[Proof of Proposition \ref{prop:robust}]
Let $m\defined |\calX||\calY|\calD|$. The distribution $p_{D,X,Y}$ is the type \citep{cover_elements_2006}[Chap. 11] of $n$ observations of $q_{D,X,Y}$. Then\footnote{Other bounds on the KL-divergence between an observed type and its distribution could be used, such as \citep{cover_elements_2006}[Thm. 11.2.2], without changing the asymptotic result.}, from \citep{csiszar_information_2004}[Corollary 2.1], for $\tau>0$
\begin{align*}
\Pr\left(\DKL(p_{D,X,Y}\|q_{D,X,Y})\geq \tau \right)&\leq \binom{n+m-1}{m-1}e^{-n\tau}\\
&\leq \left(\frac{e(n+m)}{m} \right)^me^{-n\tau}.
\end{align*}
From the Data Processing Inequality for KL-divergence,  if $\DKL(p_{D,\Yh}\|q_{D,\Yh})\leq \DKL(p_{D,X,Y}\|q_{D,X,Y})$, and, consequently,
\begin{align*}
\Pr\left(\DKL(p_{D,\Yh}\|q_{D,\Yh})\leq \tau \right)&\geq \Pr\left(\DKL(p_{D,X,Y}\|q_{D,X,Y})\leq \tau \right)\\
&\geq 1-\left(\frac{e(n+m)}{m} \right)^m e^{-n\tau}.
\end{align*}
If $\DKL(p_{D,\Yh}\|q_{D,\Yh})\leq \tau$, then since $0\leq \DKL(p_{D}\|q_{D})$, we have
\begin{align*}
\DKL(p_{\Yh|D}(\cdot|d)\|q_{\Yh}(\cdot|d))\leq \frac{\tau}{p_D(d)}~\forall d\in \calD.
\end{align*}
Choosing 
\begin{equation}
\tau = \frac{1}{n}\log \left(\frac{1}{\beta}\left( \frac{e(n+m)}{m}\right)^{m} \right),
\label{eq:deltapick}
\end{equation}
then, with probability $1-\beta$, for all $d\in D$
\begin{align*}
&\DKL(p_{\Yh|D}(\cdot|d)\|q_{\Yh}(\cdot|d))\\
&\hspace{.5in}\leq {\frac{1}{np_D(d)}}\log \left(\frac{1}{\beta}\left( \frac{e(n+m)}{m}\right)^{m} \right).
\end{align*}
Assuming that $m$,  and $c_m\defined \min_{y\in \calY,d\in \calD} p_{D,\Yh} (d,y)>0$ constant, from the proof of Lemma \ref{lem:keyLemma} and, more specifically, inequalities \eqref{eq:finalkey}, as long as $\tau \leq \min_{d,y}\frac{p_{\Yh,D}(y,d)(1-p_{\Yh|D}(y|d))}{3(1+p_{\Yh|D}(y|d))^2}$, 
\begin{align}
(1-\epsilon)\exp\left(-h(n,\beta)\right) &\leq\frac{q_{\Yh|D}(y|d)}{p_{Y_T(y)}} \\
&\leq(1+\epsilon)\exp\left(h(n,\beta)\right),
\end{align}
where 
\begin{equation}
h(n,\beta)\defined \sqrt{\frac{3}{nc_m}\log \left(\frac{1}{\beta}\left( \frac{e(n+m)}{m}\right)^{m} \right)}.
\end{equation}
Observe that $$h(n,\beta) = \Theta\left(\sqrt{\frac{1}{n}\log\frac{n}{\beta}} \right).$$ Since for $x$ sufficiently small $e^x\approx 1+x$, we have 
\begin{equation}
\frac{\left| q_{\Yh|D}(y|d)-p_{Y_T}(y)\right|}{p_{Y_T}(y)}\leq \epsilon +\Theta\left(\sqrt{\frac{1}{n}\log\frac{n}{\beta}}\right),
\end{equation}
proving the first claim.

For the second claim, we start by applying the triangle inequality:
\begin{align}
\Delta\left(q_{X,Y},q_{\Xh,\Yh}\right) \leq&\; \Delta\left(p_{X,Y},p_{\Xh,\Yh}\right) + \Delta\left(q_{X,Y},p_{X,Y}\right) \nonumber\\
&+\Delta\left(q_{\Xh,\Yh},p_{\Xh,\Yh}\right) \nonumber\\
\leq&\; \mu + \Delta\left(q_{X,Y},p_{X,Y}\right) \nonumber\\
&+\Delta\left(q_{\Xh,\Yh},p_{\Xh,\Yh}\right). \label{eq:stepTV}
\end{align}
Now assume $\DKL(p_{D,X,Y}\|q_{D,X,Y})\leq \tau$. Then the Data Processing Inequality for KL-divergence yields $\DKL(p_{X,Y}\|q_{X,Y})\leq \tau$ and $\DKL(p_{\Xh,\Yh}\|q_{\Xh,\Yh})\leq \tau$. In addition, from Pinsker's inequality,
\begin{align*}
\Delta\left(q_{X,Y},p_{X,Y}\right)&\leq 2\sqrt{2\DKL(p_{X,Y}\|q_{X,Y})}\\
&\leq 2\sqrt{2\tau},
\end{align*}
and, analogously, $\Delta\left(q_{\Xh,\Yh},p_{\Xh,\Yh}\right)\leq 2\sqrt{2\tau}$. Thus \eqref{eq:stepTV} becomes
\begin{equation}
\Delta\left(q_{X,Y},q_{\Xh,\Yh}\right) \leq \mu + 4\sqrt{2\tau}.
\end{equation}
Selecting $\tau$ as in \eqref{eq:deltapick}, then, with probability $1-\beta$,
\begin{equation}
\Delta\left(q_{X,Y},q_{\Xh,\Yh}\right) \leq \mu + 4\sqrt{\frac{2}{n}\log \left(\frac{1}{\beta}\left( \frac{e(n+m)}{m}\right)^{m} \right)},
\end{equation}
and the result follows.
\end{proof}


\bibliography{references}
\bibliographystyle{icml2017}

\end{document}